\documentclass{article}

\usepackage{microtype}
\usepackage{graphicx}
\usepackage{booktabs} 

\usepackage{times}  
\usepackage{helvet}  
\usepackage{courier}  
\usepackage{url}  

\usepackage{amsmath}
\usepackage{amsfonts}
\usepackage{dsfont}
\usepackage[ruled,vlined,algo2e]{algorithm2e}
\usepackage{sansmath}
\usepackage{CJK}
\usepackage{caption}
\usepackage{color}
\usepackage{booktabs}
\usepackage{natbib}
\usepackage{amsthm}

\usepackage{algorithm}
\usepackage{algorithmic}

\usepackage{eqparbox}
\renewcommand{\algorithmiccomment}[1]{\hfill\eqparbox{COMMENT}{\# #1}}

\theoremstyle{definition}
\newtheorem{definition}{Definition}[section]
\newtheorem{theorem}{Theorem}[section]

\usepackage[accepted]{icml2019}

\icmltitlerunning{Reward Shaping via Meta-Learning}

\begin{document}

\twocolumn[
\icmltitle{Reward Shaping via Meta-Learning}



\icmlsetsymbol{equal}{*}

\begin{icmlauthorlist}
\icmlauthor{Haosheng Zou}{equal,to}
\icmlauthor{Tongzheng Ren}{equal,to}
\icmlauthor{Dong Yan}{to}
\icmlauthor{Hang Su}{to}
\icmlauthor{Jun Zhu}{to}
\end{icmlauthorlist}

\icmlaffiliation{to}{Dept. of Comp. Sci. \& Tech., State Key Lab of Intell. Tech. \& Sys., TNList Lab, CBICR Center \\
Tsinghua University, Beijing, China}

\icmlcorrespondingauthor{Jun Zhu}{dcszj@mail.tsinghua.edu.cn}
\icmlcorrespondingauthor{Haosheng Zou}{zouhs16@mails.tsinghua.edu.cn}
\icmlcorrespondingauthor{Tongzheng Ren}{rtz19970824@gmail.com}

\icmlkeywords{Machine Learning, ICML}

\vskip 0.3in
]



\printAffiliationsAndNotice{\icmlEqualContribution} 

\begin{abstract}
Reward shaping is one of the most effective methods to tackle the crucial yet challenging problem of \emph{credit assignment} in Reinforcement Learning (RL). However, designing shaping functions usually requires much expert knowledge and hand-engineering, 
and the difficulties are further exacerbated given multiple similar tasks to solve.
In this paper, we consider reward shaping on a distribution of tasks, and propose a general meta-learning framework to automatically learn the efficient reward shaping on newly sampled tasks, assuming \emph{only} shared state space but not necessarily action space. 
We first derive the theoretically \emph{optimal} reward shaping in terms of credit assignment in model-free RL. We then propose a value-based meta-learning algorithm to extract an effective prior over the optimal reward shaping. The prior can be applied directly to new tasks, or \emph{provably} adapted to the task-posterior while solving the task within few gradient updates.
We demonstrate the effectiveness of our shaping 
through significantly improved
learning efficiency and interpretable visualizations across various settings, including notably a successful transfer from DQN to DDPG.

\vspace{-2ex}
\end{abstract}

\section{Introduction}
Reinforcement Learning (RL) has recently attracted much attention with its success in various domains such as Atari \citep{mnih2015human} and Go \citep{silver2017mastering}. 
However, the problem of \emph{credit assignment} \citep{minsky1961steps} still troubles its learning efficiency.
It is rather difficult for RL agents to answer the following question: 
how to distribute credit for success (or penalty for failure) among the sequence of decisions involved in producing the result from naturally delayed (even sparse) rewards. If the agent could know exactly which actions are right or wrong, RL would be no more difficult than supervised learning. 
Such inefficiency in credit assignment is one major reason for the unsatisfactory learning efficiency of current model-free RL methods.


Reward shaping is one of the most intuitive, popular and effective
solutions to credit assignment, whose very goal is to shape the original delayed rewards to properly reward or penalize intermediate actions as in-time credit assignment.
The technique first emerges in animal training~\citep{skinner1990behavior}, and is then introduced to RL \citep{dorigo1994robot, mataric1994reward} to tackle increasingly complex problems like Doom \citep{wu2016training} and Dota 2 \cite{openaifive}. 
While most shaping functions could be directly applied,
it is proved that optimal policies remain invariant under certain ones, namely potential-based shaping functions \citep{ng1999policy}. 

However, almost all reward shapings are \emph{hand-crafted} and need to be carefully designed by experienced human experts \citep{wu2016training,openaifive}. On one hand, coding those shaping functions in programming languages is potentially tedious and inconvenient especially in complex large-scale environments such as Doom \citep{wu2016training} and Dota 2 \citep{openaifive}. On the other hand, humans have to theoretically justify the shaping rewards to ensure that they lead to expected behavior but not other local optima. 
Together this makes effective reward shapings hard to design/code, and easily coded shapings usually ineffective.

Furthermore, in practice we are usually interested in solving multiple similar tasks as a whole. For example, when training an RL agent to solve 2D grid mazes, we wouldn't like to train individual agents for each maze map, but would naturally hope for \emph{one} general agent for all possible mazes. 
The shared but not identical task-structures naturally induce a \emph{distribution over tasks}, which in this case is a distribution over maze configurations \cite{wilson2007multi} and could elsewhere be a distribution over system parameters \cite{lazaric2010bayesian} for different robot-hand sizes or over game maps for RTS games \cite{jaderberg-ctf-2018}.
The ability to quickly solve new similar tasks drawn from such distributions is much expected for general intelligence, since it is mastered by human infants quite young \cite{smith2017developmental}.
However, the human effort in reward shaping would be further exacerbated, where we have to either design a different shaping per task or come up with a general task-dependent function presumably harder to design.


To this end, we consider the generally hard problem of reward shaping on a distribution of tasks.
Motivated by the inconvenience in reward shaping under task multiplicity, we seek to design a \emph{general, automatic} reward shaping mechanism that works well on the task distribution without hand-engineering of human experts.
We first derive the theoretically \emph{optimal} reward shaping in terms of credit assignment in model-free RL to be the optimal V-values.
By spotting that there exists shared knowledge across tasks on the same distribution,
we then propose a novel value-based algorithm based on Model-Agnostic Meta-Learning (MAML) \citep{finn2017one},
leveraging meta-learning to extract such prior knowledge.
This \textbf{prior} approximates the optimal potential-based shaping function \citep{ng1999policy} for each task. The meta-learned prior conducts reward shaping on newly sampled tasks either directly (zero-shot) or adapting to the \textbf{task-posterior} optimum (few-shot) to shape rewards in the meantime of solving the task. We provide theoretical guarantee for the latter.
Extensive experiments demonstrate the effectiveness of our reward shaping in both two cases.


To summarize, our contributions are:
(1) We present a first attempt to conduct \emph{general, automatic} reward shaping with meta-learning on a distribution of tasks
for better credit assignment and learning efficiency; (2) Our framework requires \emph{only} a shared state space across tasks, and could be applied either directly or adaptively on newly sampled tasks, which is quite general and flexible compared with most existing meta-learning methods and multi-task reward shaping works; (3) We theoretically derive and analyze the optimal reward shaping (w.r.t. credit assignment based on potential functions \citep{ng1999policy}) and our shaping algorithm.

\section{Preliminaries}
\label{sec:pre}

We consider the setting of multi-task reinforcement learning (RL), where the tasks follow a distribution $p(\mathcal{T})$. Each sampled task $\mathcal{T}_i\sim p(\mathcal{T})$ is a standard
Markov Decision Process (MDP) 
$M_i = (\mathcal{S}, \mathcal{A}_i, T_i, \gamma, R_i)$, where $\mathcal{S}$ is the state space, assumed to be shared by all tasks, $\mathcal{A}_i$ is the action space, $T_i:\mathcal{S}\times \mathcal{A}_i \times \mathcal{S}\to [0, 1]$ is the state transition probability, $\gamma\in [0, 1]$ is the discount factor and $R_i:\mathcal{S}\times \mathcal{A}_i \times \mathcal{S}\to \mathbb{R}$ is the reward function. Here, we use the subscript $i$ to denote that the tasks may have \emph{different} action spaces $\mathcal{A}_i$, different transition probabilities $T_i$ and different reward functions $R_i$. 

In this section, we briefly introduce the techniques on which our method is based, namely general Q-learning variants to solve individual MDPs, reward shaping functions to accelerate learning with theoretical guarantees, and meta-learning to tackle reward shaping on task distributions.

\subsection{Q-Learning}

Given any MDP $M$, a policy is a distribution $\pi(a | s)$.
The \textbf{V-value} $V_{M}^{\pi}(s)$ and \textbf{Q-value} $Q_{M}^{\pi}(s, a)$ are correspondingly defined for $(M, \pi)$ as cumulative rewards.
The goal of standard RL on a single task is to find the optimal $\pi$ that gives maximal V-(and Q-)values: $V_{M}^{*}(s) =\sup_{\pi} V_{M}^{\pi}(s), Q_{M}^{*}(s, a) =\sup_{\pi} Q_{M}^{\pi}(s, a)$.


Q-Learning \citep{watkins1992q} provides one solution to directly learn $Q_M^*$ and induce $\pi$ from it.
Different from previously tabular representations,
Deep Q-Network (DQN) \citep{mnih2015human} parameterizes the Q-value with a neural network $Q_{\theta}$ and minimizes the temporal difference (TD) error \cite{sutton1998reinforcement} with gradient descent:
\begin{equation*}
    \min_{\theta} \|R(s, a, s^\prime) + \gamma \max_{a^\prime} Q_{\theta}(s^\prime, a^\prime) - Q_{\theta}(s, a)\|^2,
    \label{dqn-obj}
\end{equation*}
where $\theta$ represents the parameters of the neural network. A periodic target network is usually adopted.

Dueling-DQN~\cite{wang2016dueling} specifically parameterizes $Q_{\theta}$ as $Q_{\theta}(s, a) = V_{\theta}(s) + A_{\theta}(s, a)$ so as to `` generalize learning across actions'' for better learning efficiency and performance. The neural network's penultimate layer outputs a V-value head $V_\theta$ and an advantage head $A_{\theta}$ that sum to the ultimate Q-value.
Still, the delayed (or even sparse) nature of rewards poses great challenge on learning.

\subsection{Potential-based shaping function}
\label{policy-invariance}
A reward-shaping function $F:\mathcal{S}\times \mathcal{A}\times \mathcal{S}\to \mathbb{R}$ modifies the original reward function and attempts to make RL methods (e.g., Q-learning) converge faster with more ``instructive'' rewards.
It generally resides in the same functional space as the reward function $R$, and transforms the original MDP $M=(\mathcal{S}, \mathcal{A}, T, \gamma, R)$ into another \textbf{shaped} MDP $M^\prime=(\mathcal{S}, \mathcal{A}, T, \gamma, R^\prime=R+F)$.
Of all possible shapings, potential-based shaping functions \citep{ng1999policy} retain the optimal policy, as summarized below.
\theoremstyle{definition}
\begin{definition}[Potential-based shaping function \citep{ng1999policy}]
{\it $F:\mathcal{S}\times \mathcal{A}\times \mathcal{S}\to \mathbb{R}$ is a potential-based shaping function if there exists a real-valued function $\Phi:\mathcal{S}\to\mathbb{R}$, such that $\forall (s, a, s^\prime )\in \mathcal{S} \times \mathcal{A} \times \mathcal{S}$,
\begin{equation*}
    F(s, a, s^\prime) = \gamma \Phi(s^\prime) - \Phi(s).
\end{equation*}
$\Phi(s)$ is thus called the potential function.}
\end{definition}
%
\begin{theorem}[Policy Invariance under Reward Shaping \citep{ng1999policy}]
\label{thm:policy_invariance}
{\it The condition that $F$ is a potential-based shaping function is necessary and sufficient for it to guarantee consistency with the optimal policy. Formally, for $M=(\mathcal{S}, \mathcal{A}, T,\gamma, R)$ and $M^\prime = (\mathcal{S}, \mathcal{A}, T,\gamma, R + F)$, if $F(s, a, s^\prime) = \gamma \Phi(s^\prime) - \Phi(s)$, then $\forall (s,a)\in \mathcal{S}\times \mathcal{A}$
\begin{equation}
    \begin{split}
         Q^{*}_{M^\prime}(s, a) =& Q^{*}_{M}(s, a) - \Phi(s),\\
         V^{*}_{M^\prime}(s) =& V^{*}_{M}(s) - \Phi(s),\\
    \end{split}
    \label{relation}
\end{equation}
so the policy derived from $Q^*_{M^\prime}$ remains the same.}
\end{theorem}

Consequently, if we choose $\Phi(s) = V^{*}_{M}(s)$, then $V^{*}_{M^\prime}(s) \equiv 0$, and ``all that would remain to be done would be to learn the non-zero Q-values'' \citep{ng1999policy}. 

However, why are the ``non-zero Q-values'' easier to learn for RL?
Agents could never know \emph{a priori} which actions' Q-values are zero, and we cannot directly induce policies from V-values without access to the underlying MDP model.
We found that the true advantage this particular reward shaping brings about is under-appreciated in previous works, 
and in Sec. \ref{optimal_shaping} we provide formal analysis and identify its theoretical optimal efficiency in \emph{credit assignment}, motivating our framework based on such shaping functions.

\subsection{Meta-Learning}

Meta-learning is an effective strategy to deal with a distribution of tasks. 
Specifically, it operates on two sets of tasks: \textbf{meta-training} set $\{\mathcal{T}_i\}_{i=1}^N$ and \textbf{meta-testing} set $\{\mathcal{T}_j\}_{j=N+1}^{N+M}$, both drawn from the same task distribution $p(\mathcal{T})$. The meta-learner attempts to learn the structure of tasks during meta-training, and in meta-testing, it leverages the structure to learn efficiently on new tasks with a limited number of newly observed examples from new tasks.


Meta-learning methods have been developed in both supervised learning \citep{santoro2016meta,vinyals2016matching} and RL settings \cite{duan2016rl,wang2016learning}. One of the most popular algorithms 
is 
Model-Agnostic Meta-Learning (MAML) \citep{finn2017model}, which meta-learns an versatile initialization $\theta$ of model parameters by:
\begin{align}
    \phi_i \leftarrow & \theta - \alpha \nabla_{\theta}\mathcal{L}_{\mathcal{T}_i}(f_{\theta}), \label{eqn:inner} \\
    \theta \leftarrow & \theta - \beta \nabla_{\theta} \mathbb{E}_{\mathcal{T}_i} \mathcal{L}_{\mathcal{T}_i} (f_{\phi_i}).
    \label{eqn:outer}
\end{align}
where $\theta$ are the parameters to be learned, $\phi_i$ are the task-specific parameters updated from $\theta$ as initialization (Eqn. \eqref{eqn:inner}),
$\alpha$ and $\beta$ are learning rates and $\mathcal{L}_{\mathcal{T}_i}$ is the loss function on each $\mathcal{T}_i$. Note that $\phi_i$ depend on $\theta$ and the gradients back-propagated through $\phi_i$ to $\theta$ (Eqn. \eqref{eqn:outer}).
In meta-testing, given data from the new task $\mathcal{T}_j$,
MAML adapts model parameters starting from $\theta$.
MAML has also been recently extended to a more Bayesian treatment \cite{grant2018recasting,yoon2018bayesian,NIPS2018_8161}.

\section{Methods}
\label{sec:method}
Based on the notions and notations in Sec.~\ref{sec:pre}, we first formulate the problem of learning shaping functions on a distribution of tasks. Then we derive the optimal shaping function we'd like to learn and introduce our algorithm to learn the shaping function on sampled tasks from the distribution. Lastly we introduce how to use the learned shaping function on newly sampled tasks.

\subsection{Problem Formulation}

Our goal is to learn a potential function $\Phi(s):\mathcal{S} \to \mathbb{R}$ capable of effective reward shaping on tasks sampled from the distribution to accelerate their learning. 
We seek to learn $\Phi(s)$ via mete-learning on a certain number of sampled tasks. In terms of meta-learning, this is the \emph{meta-training} phase to extract \emph{prior} knowledge from the task distribution. In light of this and recent works \cite{grant2018recasting,yoon2018bayesian,NIPS2018_8161}, we call $\Phi(s)$ the potential function \textbf{prior}.
During \emph{meta-testing} phase, we seek to directly plug in the prior to shape rewards as a general test, or to adapt it to the \textbf{task-posterior} $\Phi_i(s|\mathcal{T}_i)$ under more restricted conditions for more effective shaping.

Note that in implementation we instantiate the prior as $\Phi(s; \theta)$ and task-posterior as $\Phi_i(s|\mathcal{T}_i; \phi_i)$, i.e., ordinary neural networks rather than distributions. However, our method could still be understood from a Bayesian perspective by treating the prior as a delta function, the task-posterior as maximum-a-posteriori inference and the overall algorithm as empirical Bayes, the details of which are beyond the scope of this paper and readers may refer to \cite{grant2018recasting,yoon2018bayesian,NIPS2018_8161}.

Next, we first  derive the ideal task-posterior $\Phi(s|\mathcal{T}_i)$.

\subsection{Efficient Credit Assignment with Optimal Potential Functions}
\label{optimal_shaping}
Delving deeper into the particular potential function $\Phi(s) = V^{*}_{M}(s)$ in Sec. \ref{policy-invariance}, we first show that the substantial advantage it brings to credit assignment, which the ``non-zero Q-values'' fail to identify,
is the following:
\begin{theorem}
\label{thm:shape-opt}
{\it Shaping with $\Phi(s) = V^{*}_{M}(s)$ is optimal for credit assignment and learning efficiency.}
\end{theorem}
\begin{proof}
We first show that the reward shaping gives non-positive immediate rewards with the optimal actions' rewards exclusively zero.
To see this, consider a general MDP $M$ and the corresponding \textbf{shaped} MDP $M^\prime$, we have
\begin{align*}
    R^\prime(s, a) = & \mathbb{E}_{s^\prime}R^\prime(s, a, s^\prime)\\
    = & \mathbb{E}_{s^\prime}[R(s, a, s^\prime) + \gamma V^*_M(s^\prime) - V^*_M(s)]\\
    = & Q^*_M(s,a) - \max_a Q^*_M(s, a)\\
    \leq & 0,
\end{align*}
where the last equality holds iff $a = \arg\max_a Q_M^*(s, a)$. 
\begin{algorithm}[t]
\caption{Meta-learning potential function prior}
\begin{algorithmic}
    \label{rsml}
    \STATE \textbf{Input:} $p(\mathcal{T})$: a distribution over tasks
    \STATE \textbf{Input:} $\alpha$, $\beta$: step sizes
    \STATE \textbf{Output:} Learned prior $\theta$
    \STATE Randomly initialize parameter $\theta$ for prior
    \FOR{$\texttt{meta\_iteration} = 0, 1, 2...$}  
    \STATE Sample a batch of tasks $\mathcal{T}_i \sim p(\mathcal{T})$ 
        \FORALL{$\mathcal{T}_i$}
        \STATE Initialize replay buffer $\mathcal{D}_i$
        \STATE Collect experience $\{s_0, a_0, r_0, \cdots\}$ with $\epsilon$-greedy using $Q_{\theta}(s, a)$ and add to the replay buffer $\mathcal{D}_i$
        \STATE Evaluate $\nabla_{\theta}\mathcal{L}_{\mathcal{T}_i}(Q_{\theta})$ using samples from $\mathcal{D}_i$ ($\mathcal{L}_{\mathcal{T}_i}$ defined in Eqn. \eqref{task_objective})
        \STATE Compute adapted parameters with gradient descent: $\phi_i = \theta - \alpha \nabla_{\theta}\mathcal{L}_{\mathcal{T}_i}(Q_{\theta})$ 
        \ENDFOR
    \STATE Update $\theta \leftarrow \theta -\beta \nabla_{\theta}\mathbb{E}_{\mathcal{T}_i} \|Q_{\theta}(s, a)-Q_{\phi_i}(s, a)\|^2$ with previous samples from all $\mathcal{D}_i$
    \ENDFOR
\end{algorithmic}
\end{algorithm}

Therefore, after shaping the rewards with $\Phi(s) = V^{*}_{M}(s)$, at \emph{any} state, only the optimal action(s) give zero immediate reward, and all the other actions give strictly negative rewards right away. As a result, \emph{credit assignment} could be achieved the most efficiently since the agent could spot a deviation from the optimal policy as soon as it receives a negative reward. The optimality of any action could be determined instantaneously after it's taken without any need to consider future rewards, and any RL algorithm could penalize negative-reward actions without any fear that they might lead to better rewards in the future, 
hence the theoretically optimal efficiency in credit assignment.
\end{proof}

We thus choose $V_{M_i}^*(s)$ as the adaptation target of task-posterior  $\Phi_i(s | \mathcal{T}_i)$.
In practical RL, the non-positivity may not always hold with sampled experience and rewards from the environment, but the property still holds under expectation, and mini-batches of data approximate the very expectation.
Learning efficiency could therefore be still improved, which will be demonstrated through our experiments.


\subsection{Meta-Learning Potential Function Prior}
\label{sec:meta-learn}
The optimal shaping function $V_{M_i}^{*}(s)$
is task-specific without a universal optimum for all tasks $\mathcal{T}$. 
Inspired by MAML's idea to learn a proper prior capable of fast adaptation to the task-posterior, we propose Alg.~\ref{rsml}, as detailed below.

Formally, we specify the prior as $\Phi(s; \theta)$ defined on $\mathcal{S}$ with parameters $\theta$. For each task $\mathcal{T}_i$, the task-posterior $\Phi_i(s|\mathcal{T}_i; \phi_i)$ adapts in the direction of $V_{M_i}^{*}(s)$ initialized from $\theta$. Then, a natural objective of learning prior is:
\begin{equation}
    \label{eqn:prior-the}
    \min_{\theta}\mathbb{E}_{\mathcal{T}_i}\|\Phi(s;\theta) - V_{M_i}^{*}(s)\|^2. 
\end{equation}
%
However, $V_{M_i}^{*}(s)$ is not directly accessible for any MDP, and neither is there any RL algorithm to directly learn optimal V-values. We therefore first specify the adaptation from prior to $V_{M_i}^{*}(s)$, and then return to the learning of prior.

\begin{algorithm}[t]
\caption{Meta-testing (adaptation with advantage head)}
\begin{algorithmic}
    \label{meta-test}
    \STATE \textbf{Input:} $\mathcal{T}_j$: new task to solve
    \STATE \textbf{Input:} $\phi_j$: task-posterior parameters, initialized as learned prior $\theta$
    \STATE \textbf{Output:} adapted task-posterior $\phi_j$
    \STATE Initialize replay buffer $\mathcal{D}$
    \FOR{$\texttt{gradient\_step} = 0, 1, 2...$}
    \STATE Collect experience $\{s_0, a_0, r_0, \cdots\}$ with $\epsilon$-greedy using $A_{\phi_j}(s, a)$ and add to replay buffer $\mathcal{D}$
    \STATE Update $A_{\phi_j}(s, a)$ with Eqn. \eqref{first_step} with samples from $\mathcal{D}$ and current potential function $V_{\phi_j}(s)$ for shaping
    \STATE Update $V_{\phi_j}(s)$ with Eqn. \eqref{second_step} with samples from $\mathcal{D}$
    \ENDFOR
\end{algorithmic}
\end{algorithm}

\textbf{Task-Posterior Adaptation:}
Existing policy-based RL methods either don't estimate values or simply use the value output as baseline or bootstrap, leaving value-based RL methods more suitable for our framework.
In this paper, we simply choose Q-learning \citep{watkins1992q}, though in principle any value-based algorithm explicitly estimating optimal values could be adopted.

Q-learning still cannot directly estimate optimal V-values. To address this, we decompose the optimal values as:
\begin{equation*}
    Q_{M_i}^{*}(s, a) = V_{M_i}^{*}(s) + A_{M_i}^{*}(s, a), 
\end{equation*}
where $A_{M_i}^{*}(s, a)$ is the advantage-value function. 
%
We implement this by separating the V-value head and advantage head before the network outputs Q-value:
\begin{equation*}
    Q_{\phi_i}(s, a) = V_{\phi_i}(s) + A_{\phi_i}(s, a),
    \label{decomposition}
\end{equation*}
where $\phi_i$ is initialized as $\theta$. 

Note that $V_{\phi_i}(s)$ (and $V_{\theta}(s)$) are the potential functions we need, so $\Phi_i(s|\mathcal{T}_i; \phi_i)$ (and $\Phi(s; \theta)$) are just part of the whole network, but for completeness we denote the overall parameters $\phi_i$ (and $\theta$) and treat $\Phi_i(s|\mathcal{T}_i; \phi_i)$ (and $\Phi(s; \theta)$) as ``augmented'' potential functions.

Task-posterior adapts by following Q-learning and minimizing the TD error:
\begin{equation}
    \mathcal{L}_{\mathcal{T}_i}(Q_{\phi_i}) = \|R_i(s, a, s^\prime) + \gamma \max_{a^\prime} Q_{\phi_i}(s^\prime, a^\prime) - Q_{\phi_i}(s, a)\|^2.
    \label{task_objective}
\end{equation}
This method was first introduced in dueling-DQN \citep{wang2016dueling} but for a different purpose of speeding up training. Here we exploit the architecture in estimating the optimal V-values. To see this, first note that for identifiability of $V$ and $A$, the maximum of the output advantage function is further subtracted from $Q$ in implementation:
\begin{equation}
\label{eqn:ddqn}
    Q_{\phi_i}(s, a) = V_{\phi_i}(s) + A_{\phi_i}(s, a) - \max_{a^\prime}A_{\phi_i}(s, a^\prime). 
\end{equation}
%
As $Q_{\phi_i}$ attains $Q^*_{M_i}$ during Q-learning, by taking $\max_a$ on both sides of Eqn. \ref{eqn:ddqn}, we get
$V_{\phi_i}(s) = \max_a Q^*_{M_i}(s, a) = V^*(s)$. We can therefore learn the optimal V-values with dueling-DQN, adapting to task-posterior from prior.

\textbf{Prior Learning:}
Following the design of the task-posterior, the prior is naturally instantiated also as a dueling-DQN $Q_{\theta}(s, a) = V_{\theta}(s) + A_{\theta}(s, a)$. Similar as MAML \cite{finn2017model},
we explicitly model the desired property of the prior to be able to efficiently adapt to the task-posterior. 
Based on that each task-posterior adapts from $\theta$ to $\phi_i$ on $\mathcal{T}_i$ with $N$ steps of gradient update, we could finally rewrite the impractical prior-learning problem (\ref{eqn:prior-the}) as a practical one:
\begin{equation}
\label{eqn:prior-learning}
    \min_{\theta} \mathbb{E}_{\mathcal{T}_i}\|Q_{\theta}(s, a)-Q_{\phi_i}(s, a)\|^2. 
\end{equation}
It is worth noting that this problem is in essence different from that of DQN, as it does \emph{not} compute bootstrapped Q-values for TD error but directly uses $Q_{\phi_i}$ under the expectation of $\mathcal{T}_i$ as the learning target for $Q_{\theta}$.

Also note that in implementation we keep the full computational graph of task-posterior adaptation so $\phi_i$ is dependent on $\theta$ and gradients could back-propagate through $\phi_i$ to $\theta$. 
For all our experiments we set $N=1$ for simplicity, but $N>1$ is a natural implementational extension.
Possibly thanks to task multiplicity, we didn't find target networks necessary for the Q-networks.
Besides, since it's still an overall off-policy algorithm, we don't need to re-sample data for $\theta$ update, contrary to MAML.

\subsection{Meta-Testing with Potential Function Prior}
\label{sec:meta-test}
During meta-testing, we aim to find the optimal policy on newly sampled tasks $\mathcal{T}_j$ with reward shaping by the learned potential function prior. 
We use the meta-learned $V_{\theta}(s)$ 
to directly shape the MDP, which transforms the original MDP $M_j=(\mathcal{S}, \mathcal{A}_j, T_j, \gamma, R_j)$ into the shaped MDP $M_j^\prime = (\mathcal{S}, \mathcal{A}_j, T_j, \gamma, R_j^\prime := R_j + F)$, where $F(s, a, s^\prime) = \gamma V_{\theta}(s^\prime) - V_{\theta}(s)$. Intuitively, $V_{\theta}(s)$ provides a good estimate of $V_{M_{j}}^{*}(s)$ from meta-training on the task distribution, thus learning on $M_j^{\prime}$ can be much simpler than learning on $M_j$ as the reward shaping is close to optimal.

We identify two cases of meta-testing with our dueling-DQN-based meta-learning algorithm. 
\textbf{Shaping only} is one case where $V_{\theta}(s)$ is directly applied on new tasks without adaptation.
This applies to new tasks with different action spaces, or when the advantage head simply could not be used for some reason (e.g., constraints on the new policy).
According to Thm. \ref{thm:policy_invariance}, any RL algorithm could be used on the shaped MDP with the optimal policy unchanged.  
\textbf{Adaptation with advantage head} is the other case where the action space doesn't change and the DQN-policy is still applicable. We can then jointly adapt $V_{M_j}^{*}(s)$ to the task-posterior and find the optimal policy efficiently within a few updates, initializing the whole $\phi_j$ as $\theta$.

In the latter case, we still shape the MDP with the task-posterior being adapted. We iteratively collect experience using $A_{\phi_j}(s, a)$ with $\epsilon$-greedy and update $A_{\phi_j}(s, a)$ and $V_{\phi_j}(s)$ alternating the following two steps (step size $\alpha$):

$\circ$ Update $A_{\phi_j}(s, a)$ with sampled data from replay buffer:
\begin{align}
    \phi_j \leftarrow \phi_j - \alpha \nabla_{\phi_j}\|R^{\prime}_j(s, a, s^\prime) & + \gamma \max_{a^\prime} A_{\phi_j}(s^\prime, a^\prime) \notag\\
    & - A_{\phi_j}(s, a)\|^2.
    \label{first_step}
\end{align}
$\circ$ Update $V_{\phi_j}(s)$ with sampled data from replay buffer:
\begin{align}
    \phi_j \leftarrow \phi_j - \alpha \nabla_{\phi_j}\|& V_{\phi_j}(s) - \texttt{stop\_gradient}\big( \notag\\
    &\max_a A_{\phi_j}(s, a) +  V_{\phi_j}(s)\big)\|^2.
    \label{second_step}
\end{align}
\begin{theorem}
\label{thm:adapt}
{\it Eqn. \eqref{first_step} optimizes for the optimal policy. Eqn. \eqref{second_step} optimizes for the task-posterior $V_{M_j}^{*}(s)$.}
\end{theorem}

We defer the proof to Appx. \ref{sec:proof}. 

For faster adaptation on new tasks, we simply optimize Eqn. \eqref{first_step} and Eqn. \eqref{second_step} alternately, which we find sufficient in experiments.
We summarize such \textbf{adaptation with advantage head} in Alg.~\ref{meta-test}.

\textbf{Advantage over MAML:}
Note that in the latter case of meta-testing, 
one can directly adapt as the original MAML. However,
direct adaptation merely exploits the parameter initialization, while our Alg.~\ref{meta-test} also explicitly exploits the efficient reward shaping of the potential function prior in addition. The shaped rewards are easier for policy learning, and the adapting shaping (Eqn.~\eqref{second_step}) further boosts policy learning (Eqn.~\eqref{first_step}) immediately in the next loop.
Thus our Alg.~\ref{meta-test} is faster and more stable than direct MAML, and in Sec.~\ref{sec:exp} we compare with and outperform MAML.
We also emphasize that we only assume shared state space, facilitating adaptation across discrete and continuous action spaces, which MAML cannot achieve.



\section{Related Work}

To the best of our knowledge, 
the only recent work on \emph{automatic} reward shaping on a task distribution is \citet{jaderberg-ctf-2018}.
In addition to being independent of our work, the difference of \citet{jaderberg-ctf-2018} is that they access the limited novel states (termed ``game events'') of the game engine of their specific task and only need to evolve the rewards for those states. Such rewards are simply stored in a short, fixed-length table and optimized with evolution strategies, with the meta-optimization objective of evolution being also designed task-specifically.
Earlier similar works \cite{konidaris2006autonomous,snel2010multi} are also restricted in various ways such as relying on specific feature choice and evolution heuristics, being unable to adapt to new tasks as ours, lacking theoretical analysis of reward shaping on credit assignment or being unable to scale to complex environments with simple models.
In contrast to those works, our method is quite general, assuming no task knowledge or model access, with a more general, principled meta-learning objective, flexible application settings, novel theoretical analysis and gradient-based optimization.

Apart from \citet{jaderberg-ctf-2018}, almost all other recent RL successes in complex environments either \emph{manually} design reward shaping based on game elements, with examples in Doom \cite{wu2016training} and Dota 2 \cite{openaifive}, or simply depart from the scalar-reward RL approach and exploit rich supervision signals of other source with supervised learning \cite{dosovitskiy2016learning, silver2017mastering, Huang_2019_AAAI, wu2018hierarchical}.

\section{Experiments}
\label{sec:exp}
We demonstrate the effectiveness and generality of our framework under various settings. 
First we conduct experiments on the classic control task, CartPole \citep{barto1983neuronlike}, where the task distribution is defined varying the pole length and the action space could be either discrete or continuous. We then consider grid games whose state space is of much higher dimensionality and the maps of which hold exponential many possibilities (the task distribution is also defined on all the possible maps). 
Depending on whether the action space shares across the task distribution, the advantage head in our dueling-DQN model (and thus the Q-values) may not be applicable to newly sampled tasks. We therefore experiment under both settings to test the learning efficiency on new tasks.
Since we are more interested in general complex MDPs where shaping rewards are hard to code and our meta-training relies on function approximators to generalize on the task distribution, 
we use neural-network agents in all experiments under the model-free setting.

\subsection{Discrete and Continuous CartPoles}
\label{sec:exp-pole}

\begin{figure*}
    \parbox{0.33\textwidth}{
    \includegraphics[width=0.33\textwidth]{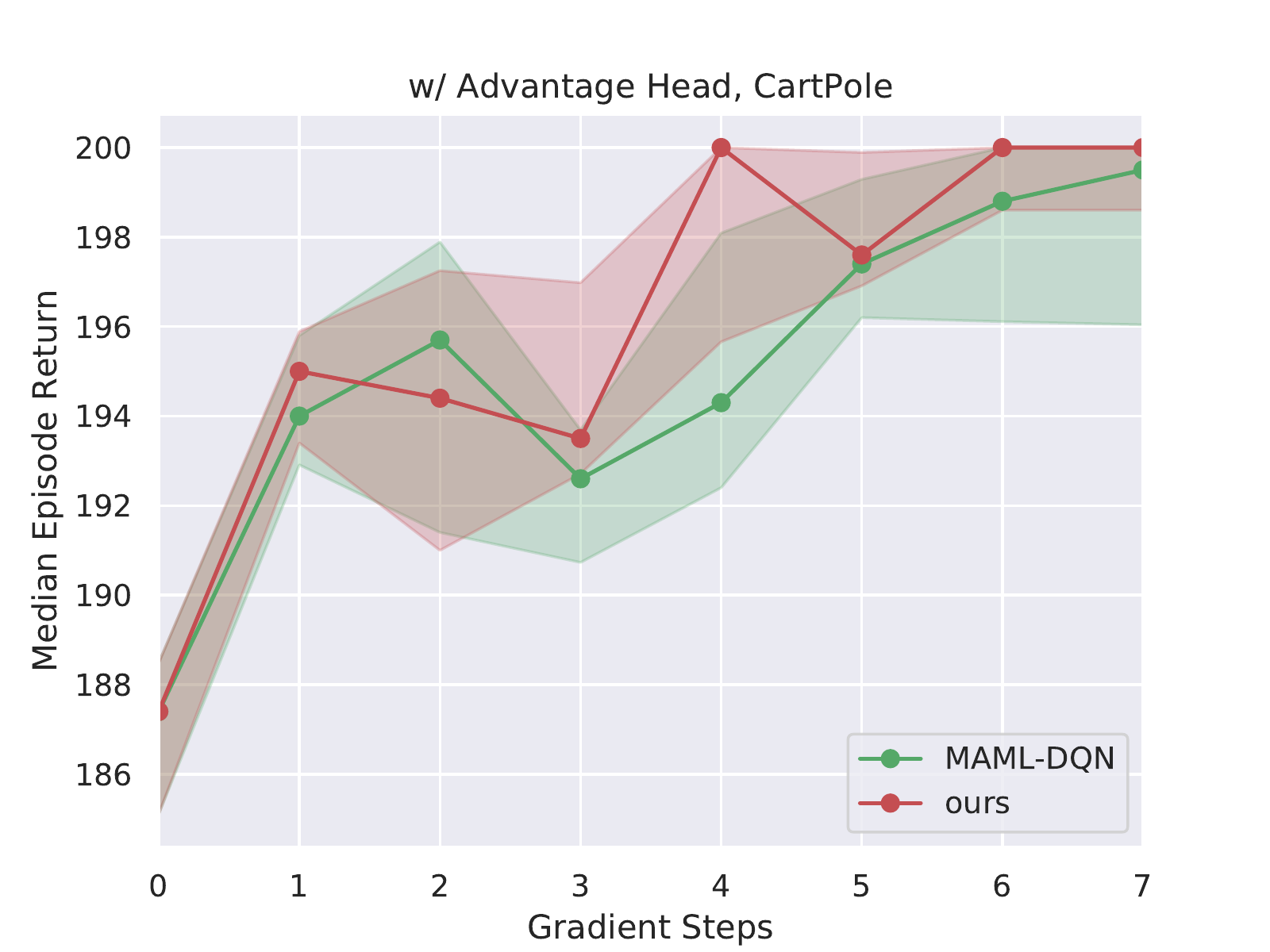}
    }
    \parbox{0.33\textwidth}{
    \includegraphics[width=0.33\textwidth]{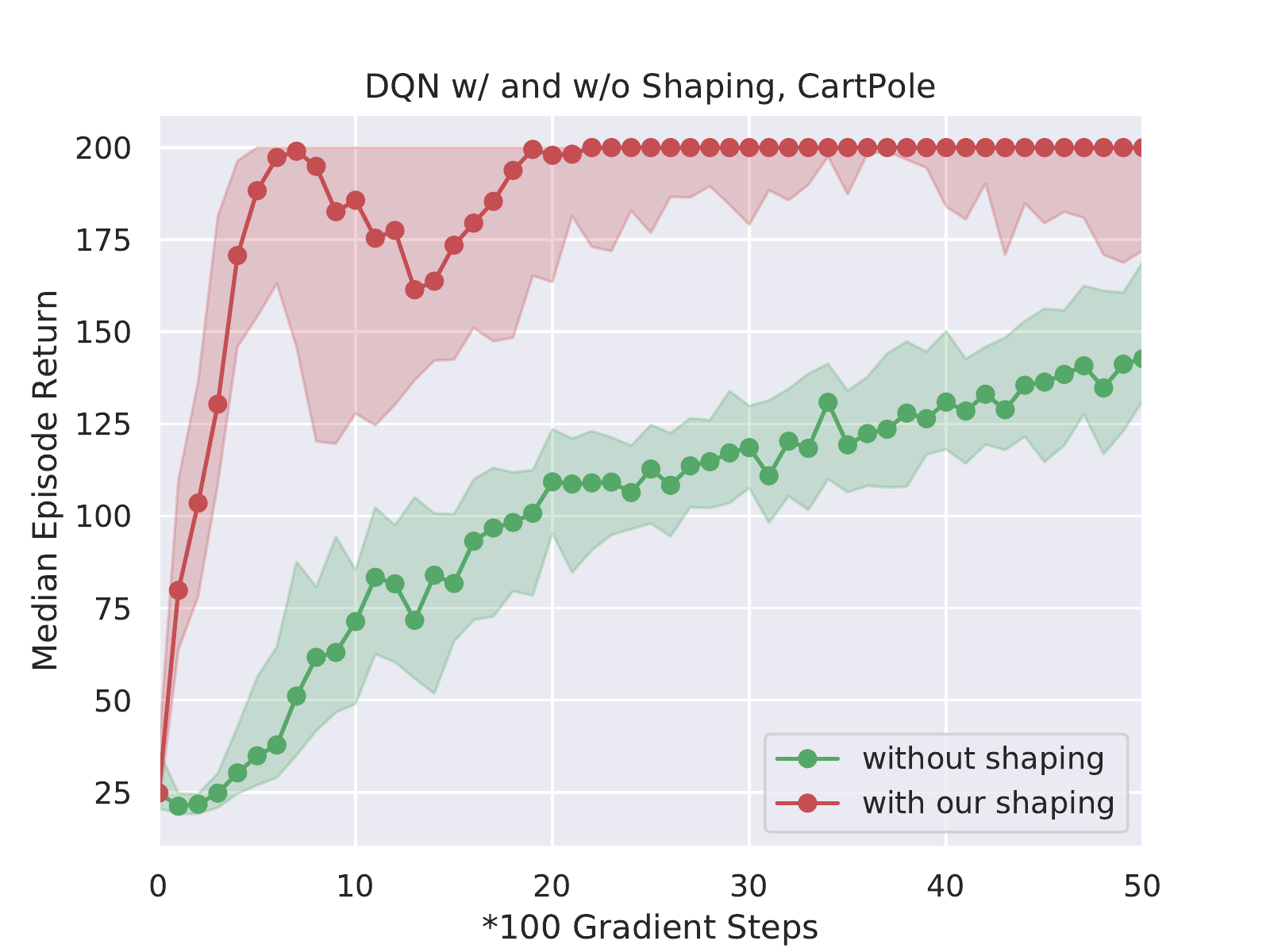}
    }
    \parbox{0.33\textwidth}{
    \includegraphics[width=0.33\textwidth]{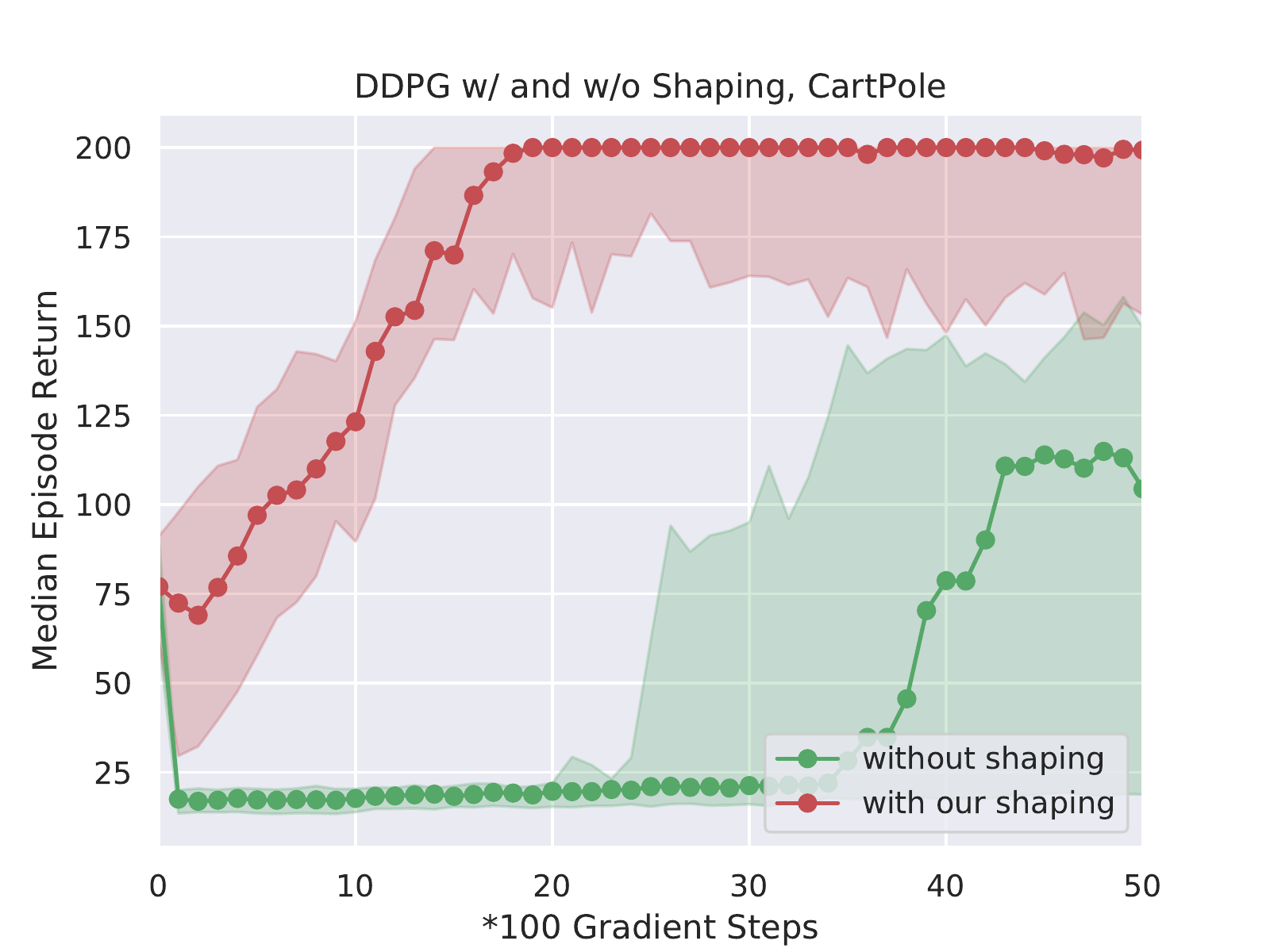}
    }
    \caption{
    Median and quartile return curves of meta-testing on newly sampled tasks on CartPoles. We uniformly outperform baselines in learning efficiency and stability.
    Left: meta-testing with advantage head. Improvement over directly applying MAML to the learned prior (green) is even clearer on later grid games.
    Middle: learning randomly initialized vanilla DQNs with (red) and without (green) our meta-learned zero-shot reward shaping.
    Right: learning randomly initialized continuous policies using DDPG with (red) and without (green) our meta-learned zero-shot reward shaping in continuous action space.
    }
    \vspace{-2ex}
    \label{fig:pole}
\end{figure*}

In CartPole \citep{barto1983neuronlike}, the agent tries to keep a pole upright by applying horizontal forces to the cart supporting the pole. Although a single particular CartPole is not very difficult, it's non-trivial to consider infinitely many CartPole tasks with different pole lengths, since the pole length affects the pole mass, mass center and, therefore, the whole dynamics of the environment. Besides, the applied forces could also be represented in either a discrete or continuous way in different tasks, posing further difficulties in solving them altogether.

A positive reward of $1$ is provided every timestep as long as the pole stays within a pre-defined ``upright'' range of 15 degrees from vertical \citep{barto1983neuronlike,1606.01540}. This reward is not sparse, but is still far from optimal in terms of \emph{credit assignment} since it does not distinguish between ``really'' upright positions and dangerous ones where the pole is yet about to fall. 
To design a properly distinguishing reward shaping obviously requires much expert knowledge of the underlying physics.
Therefore, automatic reward shaping on the distribution of CartPoles is of much significance.

\textbf{Basic Training Settings:}
We modify the CartPole environment in OpenAI Gym \citep{1606.01540} so that it accepts pole length as a construction parameter and changes the physical dynamics accordingly. The pole length is uniformly sampled within the range of $[0.25, 5.00]$ and defines a distribution over CartPoles. All the state spaces $\mathcal{S} \subset \mathbb{R}^4$. We use the discrete two-action setting (a fixed amount of force to the left or right) and the
aforementioned original reward during meta-training. Episodes terminate after 200 timesteps, so the maximum achievable return is 200.

For the dueling-DQN we use an MLP with two hidden layers of size 32, followed by one linear layer for the advantage head and one for the value head to aggregate the output Q-values as Eqn.~\eqref{eqn:ddqn}. The prior $\theta$ is meta-trained with Alg.~\ref{rsml} for 500 meta iterations with 10 sampled tasks per iteration. Note that the tasks are merely used for the meta-update in Alg.~\ref{rsml} with no performance guarantee on single tasks. All results are taken across five random seeds from $0$ to $4$.

Intuitively, Alg.~\ref{rsml} is learning to generalize over different dynamics to assess how good/bad a state is.

\textbf{Meta-Testing with Advantage Head:}
We first test the case of \textit{adaptation with advantage head} as per Sec.~\ref{sec:meta-test}, where test tasks share the action space with meta-training tasks. The meta-trained prior $\theta$ is evaluated on 40 newly sampled unseen discrete CartPoles with Alg.~\ref{meta-test} to see how fast and how well the potential function (value head), as well as the advantage head, adapts to each new task after re-initializing their weights to $\theta$. As mentioned in Sec.~\ref{sec:meta-test}, we compare with the meta-testing procedure of MAML as a baseline, keeping $\theta$ and all common hyperparameters the same. 

We track the episodic returns of the agent 
after each gradient update step, aggregate all such returns across different meta-test tasks and different runs, and plot their medians and quartiles in Fig.~\ref{fig:pole} (left). 
As can be seen,
our method performs better than MAML, achieving the max 200 two times faster (in 4 steps c.f. 8 steps) and oscillates milder, with improvement even clearer in Sec.~\ref{sec:exp-grid}. The relatively high initial return also indicates the quality of the meta-learned prior on new tasks. 
While MAML could also exploit the prior over the entire model, it's with the additional reward shaping that our method could adapt and learn on new tasks faster.
Note that oscillation could not be completely avoided since it's to some extent inherent to off-policy RL algorithms, as is shown in later experiments.

\textbf{Meta-Testing from Discrete to Continuous:}
We then test the \textit{shaping only} case as per Sec.~\ref{sec:meta-test}.
With $V_\theta$ directly used for reward shaping zero-shot, we train: (1)  a vanilla DQN with randomly initialized weights on discrete CartPoles, corresponding to situations where the advantage head could not be used, and (2) a deterministic policy network using DDPG on continuous CartPoles, corresponding to situations where meta-test tasks have \emph{different} actions spaces, which disqualifies almost all existing meta-learning methods.

The vanilla DQN has only two hidden layers of size 32 without dueling. 
It's randomly re-initialized for each test task, and we track and plot the test progress similarly as before, except that we evaluate episodic returns every 100 updates. Naturally, we compare with training the same vanilla DQN with the same common hyperparameters but without any reward shaping to test the effectiveness of the meta-learned reward shaping. As shown in Fig.~\ref{fig:pole} (middle),
the zero-shot reward shaping still significantly boost the learning process on new tasks, achieving the max 200 remarkably faster while ``without shaping'' hasn't achieved yet.

To test with continuous action, 
we further modify the CartPole environment to accept a scalar real value as action, whose sign determines the direction and absolute value determines the force magnitude. We use a deterministic policy also with two hidden layers of size 32, and an additional two-hidden layer critic network for DDPG. 
Similar as with the vanilla DQN, we run DDPG with or without our reward shaping. As shown in Fig.~\ref{fig:pole} (right), 
learning on new tasks is again significantly accelerated with our reward shaping. Note that because we apply \texttt{tanh} nonlinearity to the action output to bound the actions, the initial policy appears more stable with higher initial returns than in the discrete case. However, due to the non-optimal original reward in terms of credit assignment, DDPG without shaping confuses and struggles at first with returns dropping to below 25.

\subsection{Grid Games}
\label{sec:exp-grid}
Grid games are clean but still challenging environments for model-free RL agents in terms of navigation and planning, especially when using neural nets as the agent model \citep{tamar2016value} since tabular representations could not generalize across grids. 
Many real-world environments could be modeled as grids in 2D or 3D. While represented simple, grids could have many variations with different start and goal positions on a $8 \times 8$ grid incurring $64 \times 63 = 4032$ different tasks. Introducing additional obstacles on the maps leads to combinatorial explosion of further possibilities.

Furthermore, grids almost always come with sparse rewards with rewards obtained only in novel states like goals or traps.
Such rewards are probably the most difficult for credit assignment, 
and to manually design reward shapings requires full access to the environment model and much human knowledge and heuristics which usually pre-compute the shortest paths or some distance metrics.
Therefore, it's very important to study automatic reward shaping on the distribution of grid games. 

We randomly generate grid maps specifying start and goal positions and possibly obstacles and traps. Agents start from the start position, move in the four canonical top, down, left and right directions and only receive a positive reward of 1 upon reaching the goal. The discount factor assures that the optimal V-/Q-values display certain notion of shortest path. Episodes terminate if the agent hasn't reached the goal in certain timesteps (50 in our experiments).

We use the same representations for start, goal and obstacles respectively across different maps, so intuitively, Alg.~\ref{rsml} learns to recognize and generalize concepts of map positions and, more importantly, the notion of shortest path to goal.

\begin{figure}[t]
    \centering
    \parbox{0.48\columnwidth}{
    \includegraphics[height=2.5cm]{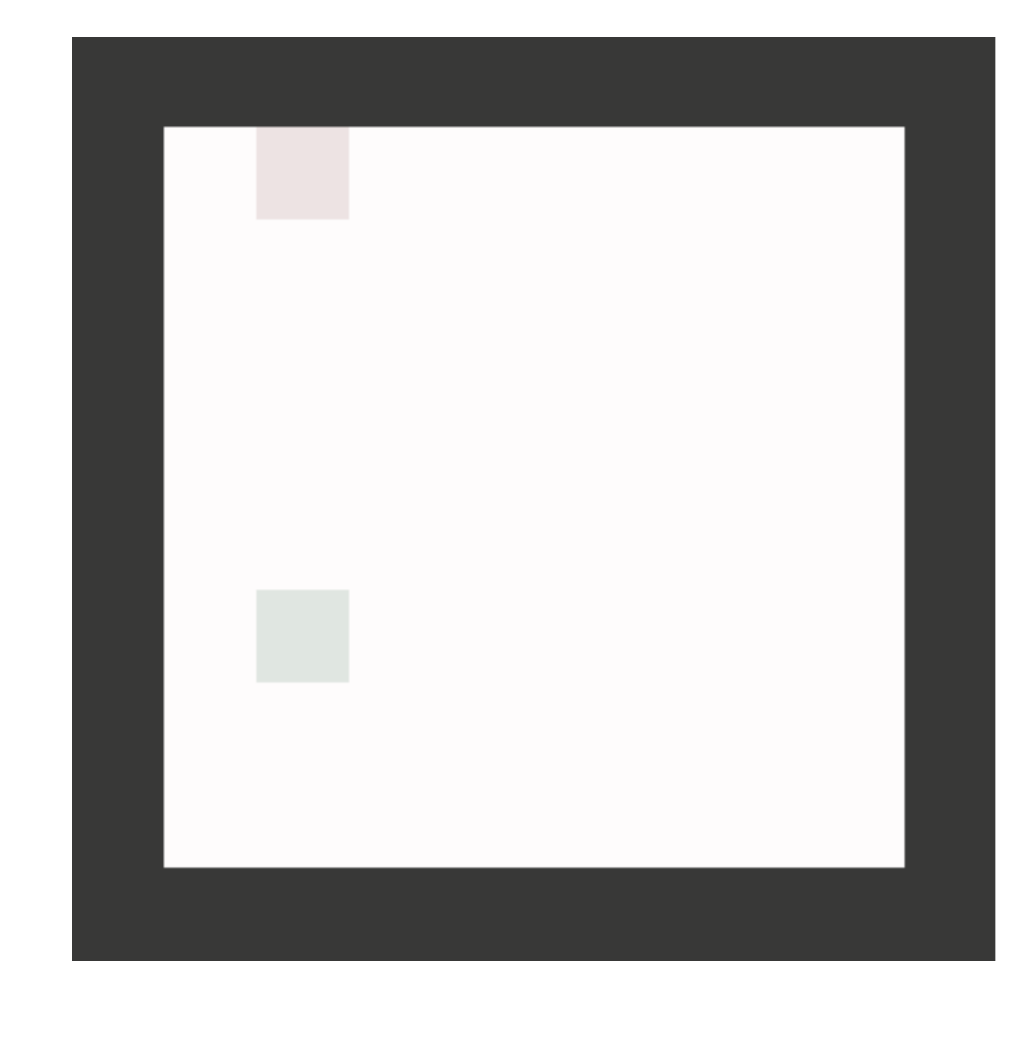}
    }
    \parbox{0.48\columnwidth}{
    \includegraphics[height=2.5cm]{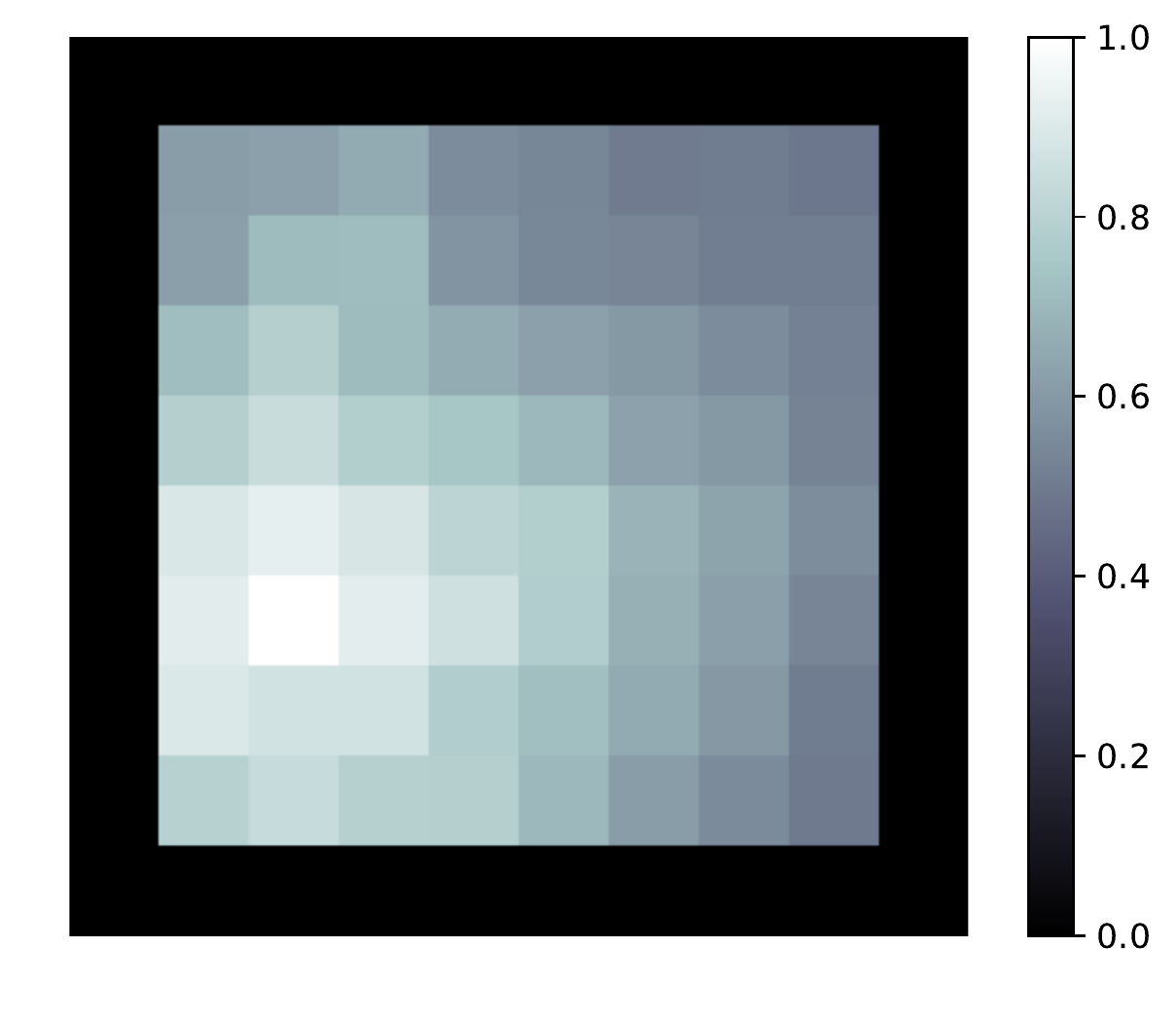}
    }
    \caption{
    Left: One instance of the clean maps with only the start position (red) and goal (green).
    Right: Visualization of the meta-learned prior (V-values) on the left map. It matches the intuition that the closer to goal, the higher the value, and is expectedly not optimal yet with values not strictly symmetric w.r.t. goal.
    }
    \label{fig:find-map}
\end{figure}

\begin{figure}[t]
    \parbox{0.48\columnwidth}{
    \includegraphics[width=0.48\columnwidth]{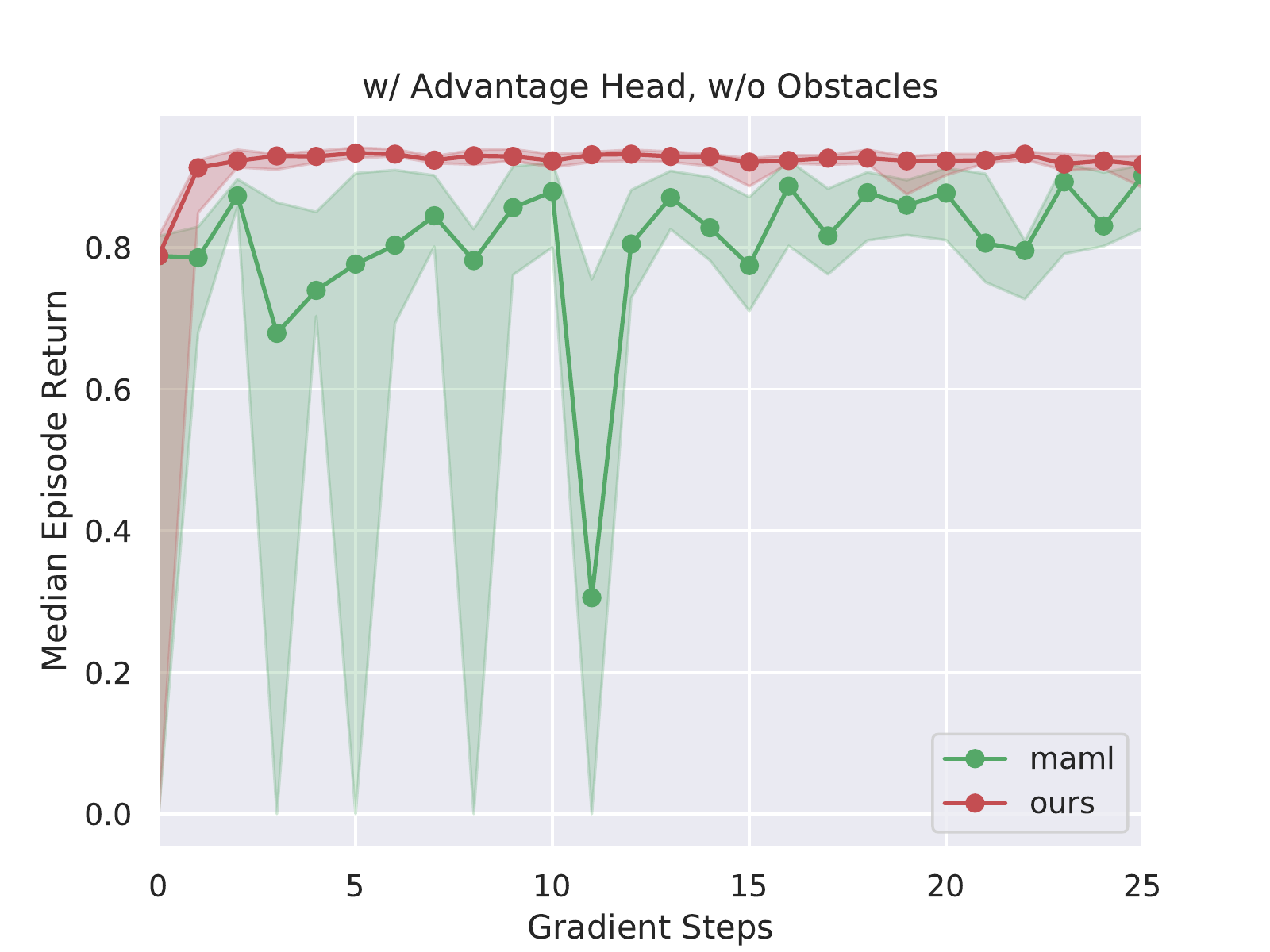}
    }
    \parbox{0.48\columnwidth}{
    \includegraphics[width=0.48\columnwidth]{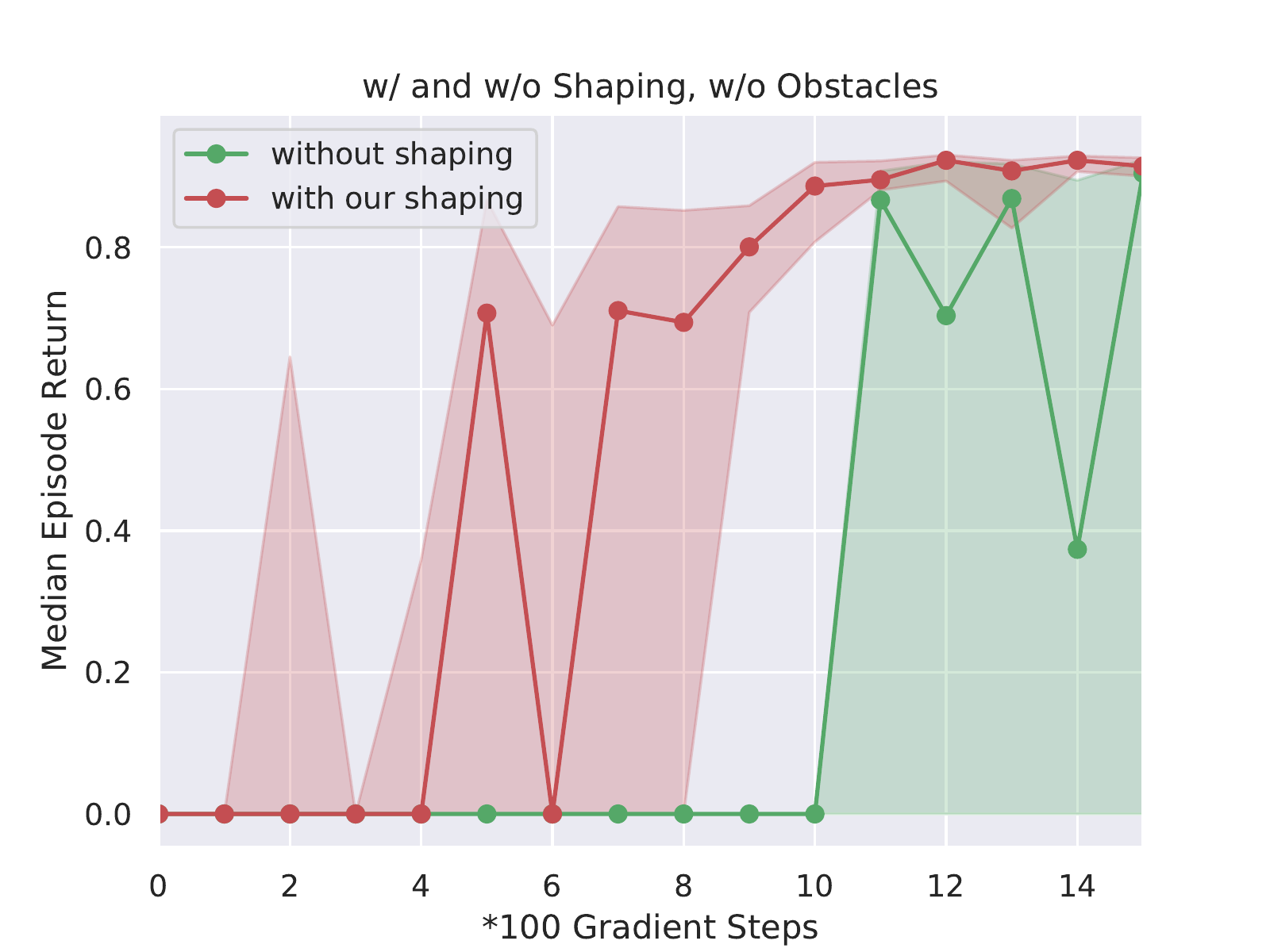}
    }
    \caption{
    Median and quartile return curves of meta-testing on newly sampled tasks on clean maps. We again uniformly outperform baselines in learning efficiency and stability.
    Left: meta-testing with advantage head.
    Right: learning randomly initialized vanilla DQNs with (red) and without (green) our meta-learned zero-shot reward shaping.
    }
    \label{fig:find}
\end{figure}

\textbf{Grid Games with Clean Maps:}
We first experimented with a simpler version of grid mazes with only start and goal positions but no obstacles. We generate 800 such maps of size $10 \times 10$ (e.g., Fig. \ref{fig:find-map} (left)) for meta-training, where all state spaces $\mathcal{S} \subset \mathbb{R}^{10 \times 10 \times 4}$ with the last dimension corresponding to the 4 channels of 0-1 maps of start, goal and current position as well as obstacles (all 0 in this case).

For the dueling-DQN we use a CNN with four convolutional layers with 32 kernels of $3 \times 3$ and stride 1, followed by two fully connected layers and then the dueling module. Meta-training is conducted similarly as in  Sec. \ref{sec:exp-pole}. 

We also meta-test the two cases as per Sec. \ref{sec:meta-test}: \textit{adaptation with advantage head} from the whole prior $\theta$, and \textit{shaping only} to train a vanilla DQN with zero-shot $V_\theta$ shaping. We mainly follow the procedure as in Sec. \ref{sec:exp-pole}, except that we don't construct continuous-action grids. All common hyperparameters are the same between any pair of our method and baseline.

As can be seen from Fig. \ref{fig:find},
our method performs much better in both cases of meta-testing in terms of learning efficiency and stability, displaying the high potential of our method in scaling to complex environments and agent models. Visualization of the learned V-values on an unseen map (Fig. \ref{fig:find-map} (right)) also justifies the meta-learned prior $\theta$.

\textbf{Grid Games with Obstacles:}
\begin{figure}[t]
    \centering
    \parbox{0.48\columnwidth}{
    \includegraphics[height=2.5cm]{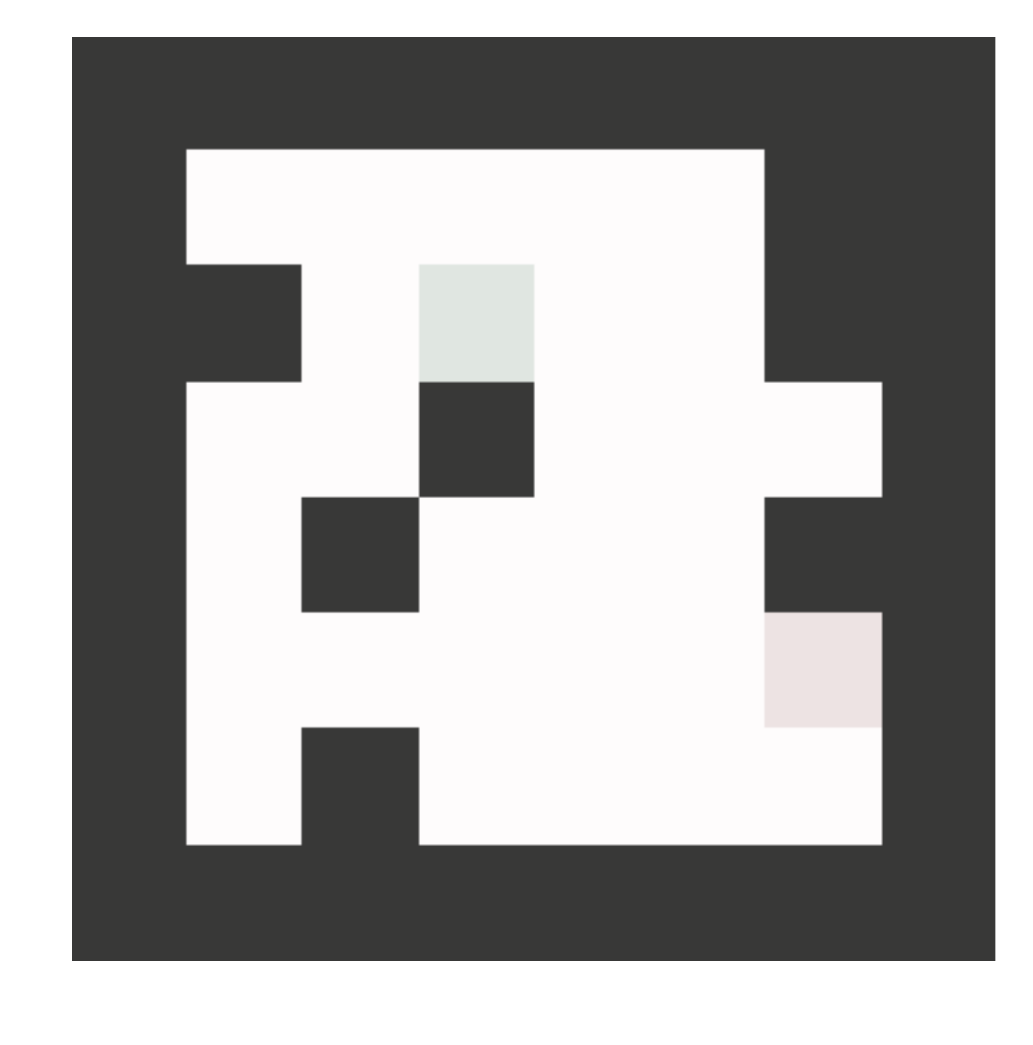}
    }
    \parbox{0.48\columnwidth}{
    \includegraphics[height=2.5cm]{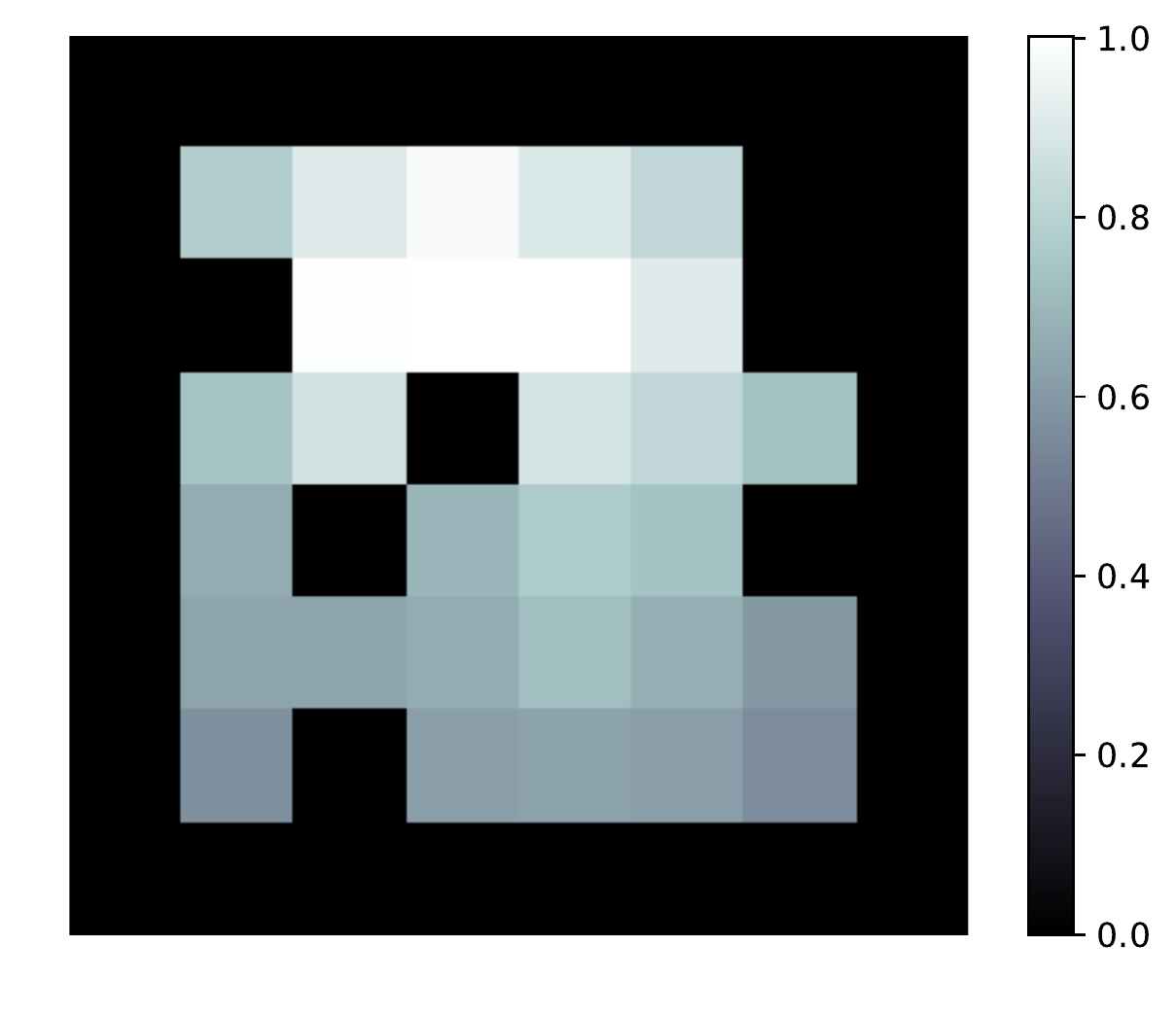}
    }

    \caption{
    Left: One instance of the maps with obstacles.
    Right: Visualization of the meta-learned prior (V-values) on the left map. It also matches the shortest path intuition with non-optimal, not strictly symmetric values.
    }
    \label{fig:maze-map}
\end{figure}
\begin{figure}[t]
    \parbox{0.48\columnwidth}{
    \includegraphics[width=0.48\columnwidth]{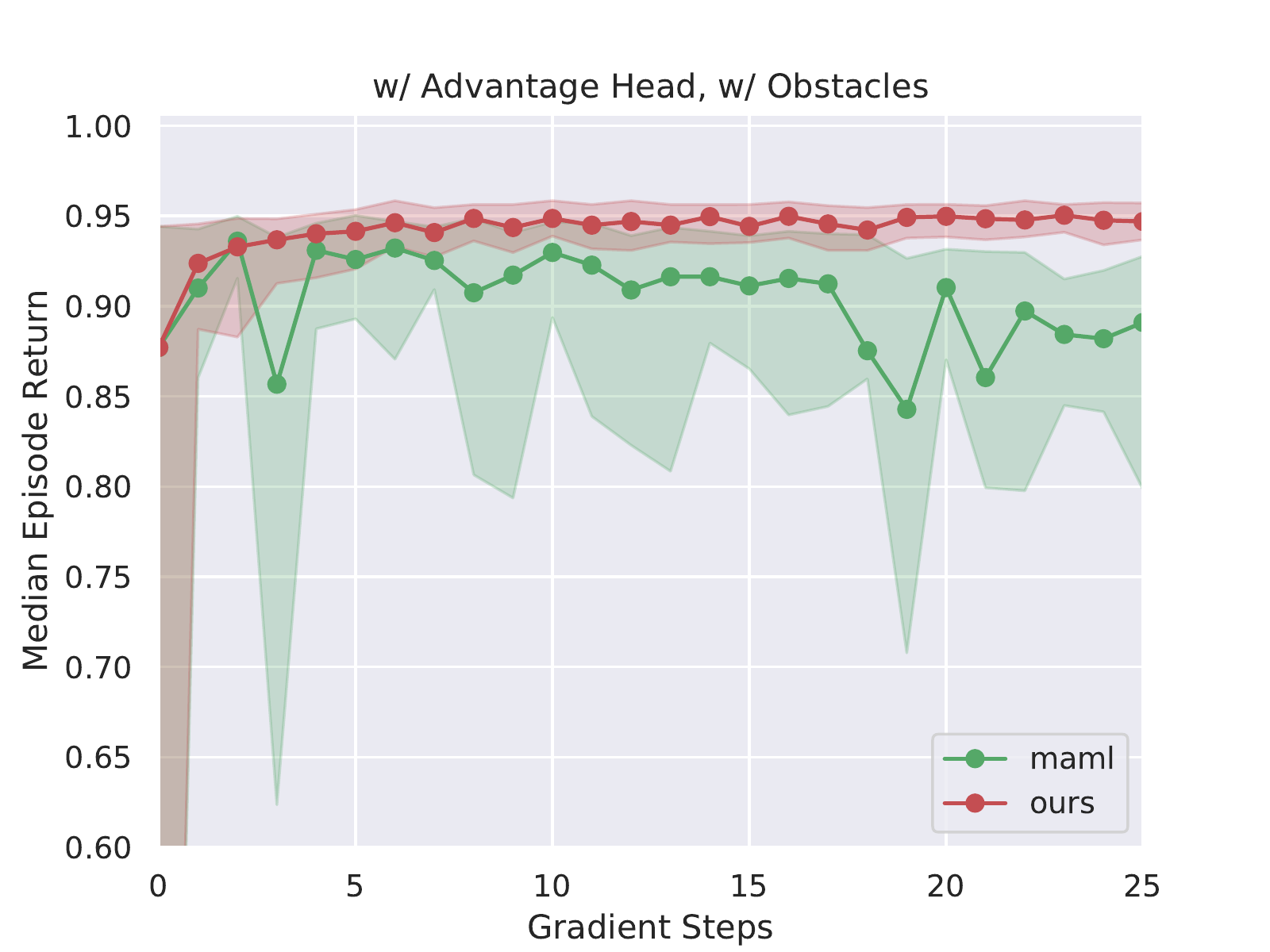}
    }
    \parbox{0.48\columnwidth}{
    \includegraphics[width=0.48\columnwidth]{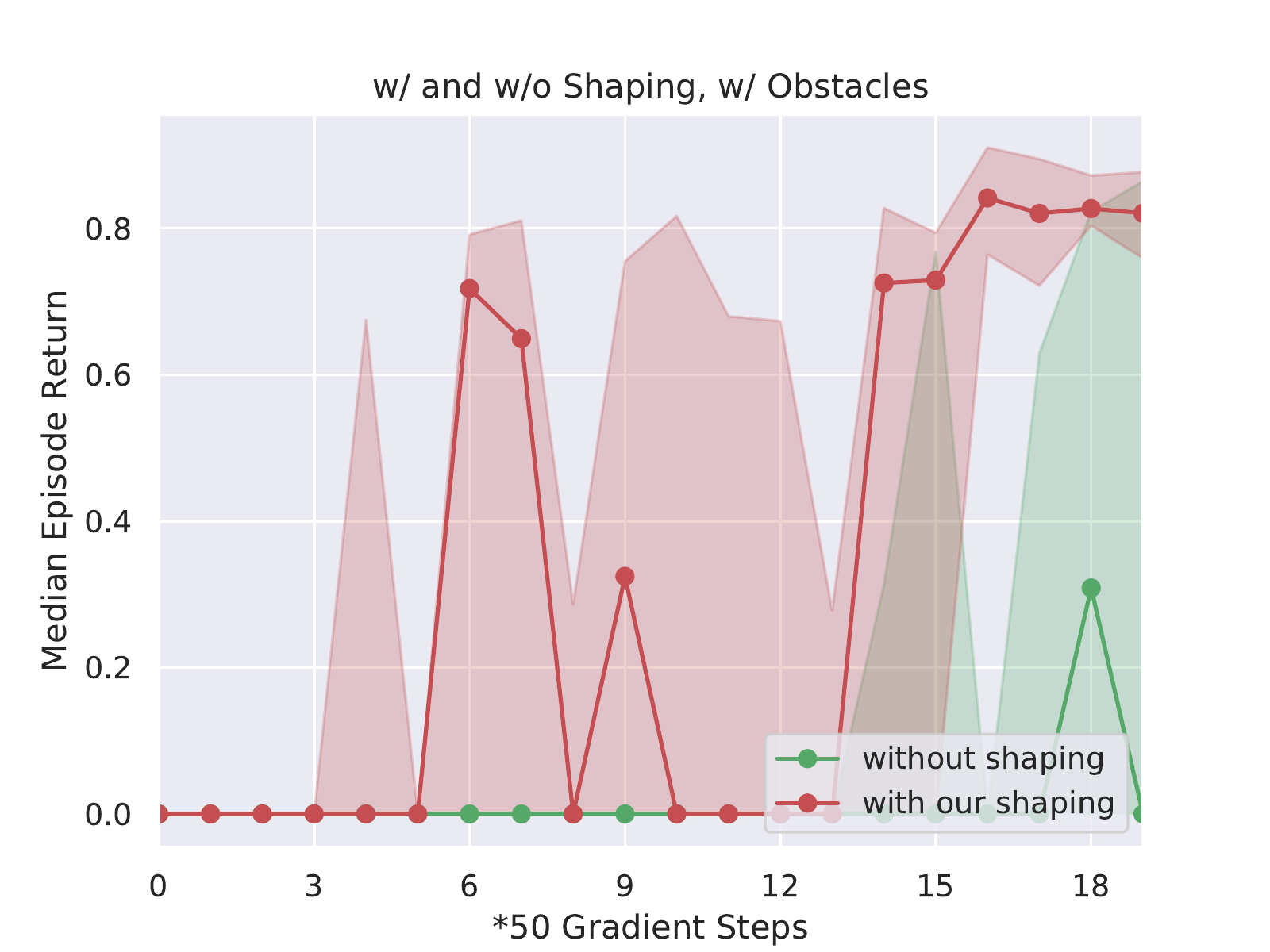}
    }
    \caption{
    Median and quartile return curves of meta-testing on newly sampled tasks on maps with obstacles. We constantly outperform baselines in learning efficiency and stability.
    Left: meta-testing with advantage head.
    Right: learning randomly initialized vanilla DQNs with (red) and without (green) our meta-learned zero-shot reward shaping.
    }
    \label{fig:maze}
\end{figure}
We then experimented with a fuller version of grid mazes where obstacles may be present at each grid position with probability 0.2 during map generation. We generate 4000 such maps of size $8 \times 8$ (e.g., Fig. \ref{fig:maze-map} (left)) for meta-training, so all state spaces $\mathcal{S} \subset \mathbb{R}^{8 \times 8 \times 4}$. We used the same convolutional dueling-DQN architecture as on clean maps, and the same meta-training/-testing protocols.

As can be seen from Fig. \ref{fig:maze},
our method constantly learns more efficiently than the baselines in both cases of \textit{adaptation with advantage head} and \textit{shaping only} on new tasks. The meta-learned $V_\theta$ of an unseen map also passes intuitive sanity check (Fig. \ref{fig:maze-map} (right)). 
Oscillation is a bit more severe than before due to the harder tasks and off-policy algorithmic nature, but ours is still superior in relative performance and stability.

\vspace{-2ex}
\section{Conclusions}
In this paper, we consider the problem of reward shaping on a distribution of tasks. We first prove the optimality of optimal V-values for potential-based reward shaping in terms of credit assignment.
We then propose a meta-learning algorithm to learn a flexible prior over the optimal V-values. The prior could be well applied directly to shape rewards and could also quickly adapt to the task-posterior optimum while solving the task. We provide additional theoretical guarantee for the latter case. Meanwhile, our framework only assumes that the state spaces of the task distribution are shared, leaving wide possibilities for potential applications.
Extensive experiments demonstrate the effectiveness of our method in terms of learning efficiency and stability on new tasks. 
We plan to consider adapting the shaping prior without the advantage head, and also single-task setting in the future.

\section*{Acknowledgement}

Haosheng Zou would personally like to thank his beautiful and sweet wife, Jiamin Deng, for her incredible suppport during the whole process of this paper by not being around most of the time.


\bibliography{main}
\bibliographystyle{icml2019}

\clearpage
\appendix
\onecolumn

\section{Proof of Thm. \ref{thm:adapt}}
\label{sec:proof}

\begin{proof}

(Part I)
{\it Eqn. \eqref{first_step} optimizes for the optimal policy.}




First, note that Eqn. \eqref{first_step}, 
$$\phi_j \leftarrow \phi_j - \alpha \nabla_{\phi_j}\|R^{\prime}_j(s, a, s^\prime) + \gamma \max_{a^\prime} A_{\phi_j}(s^\prime, a^\prime) \notag - A_{\phi_j}(s, a)\|^{2},$$
is naturally minimizing an objective in the form of TD error:

\begin{equation}
\label{eqn:first-step-obj}
    \|R^\prime_j (s, a, s^\prime) + \gamma \max_{a^\prime} A_{\phi_j}(s^\prime, a^\prime) - A_{\phi_j}(s, a)\|^2.
\end{equation}

Bearing in mind that in Alg. \ref{meta-test} we use $\Phi(s) = V_{\phi_j}(s)$ as the potential-based shaping function to obtain $R^{\prime}_j(s, a, s^\prime)$, we can rewrite objective \eqref{eqn:first-step-obj} as:
\begin{equation*}
    \begin{split}
        & \|R^\prime_j (s, a, s^\prime) + \gamma \max_{a^\prime} A_{\phi_j}(s^\prime, a^\prime) - A_{\phi_j}(s, a)\|^2 \\
        = & \|\big(R_j (s, a, s^\prime) + \gamma V_{\phi_j}(s^\prime) - V_{\phi_j}(s)\big) + \gamma \max_{a^\prime} A_{\phi_j}(s^\prime, a^\prime) - A_{\phi_j}(s, a)\|^2 \\
        = & \|R_j (s, a, s^\prime) + \gamma (\max_{a^\prime} A_{\phi_j}(s^\prime, a^\prime) + V_{\phi_j}(s^\prime)) - (A_{\phi_j}(s, a) + V_{\phi_j}(s))\|^2 \\
        = & \|R_j(s, a, s^\prime ) + \gamma \max_{a^\prime} Q_{\phi_j}(s^\prime, a^\prime) - Q_{\phi_j}(s, a)\|^2,
    \end{split}
\end{equation*}
where $A_{\phi_j}(\cdot, \cdot) + V_{\phi_j}(\cdot) = Q_{\phi_j}(\cdot, \cdot)$ simply because it's neural-network computation of the dueling architecture.

Now we've already arrived at exactly the Q-learning TD error on the original MDP $M_j$:

\begin{equation}
    \mathcal{L}_{\mathcal{T}_j}(Q_{\phi_j}) = \|R_j(s, a, s^\prime ) + \gamma \max_{a^\prime} Q_{\phi_j}(s^\prime, a^\prime) - Q_{\phi_j}(s, a)\|^2.
\end{equation}

Therefore,
Eqn. \eqref{first_step} is in essence minimizing the Q-learning TD error on $M_j$, thus optimizing for the optimal policy (invariant with/without the potential-based reward shaping). 

\textbf{Remark:}
As an alternative understanding, 
first note that $A_{\phi_j}(s, a)$ is just a notation for the neural-network head. If we view it as an estimator of $Q_{M_j^\prime}^*(s, a)$, then
Eqn. \eqref{first_step} is actually performing Q-learning on the shaped MDP $M_j^\prime$, with objective \eqref{eqn:first-step-obj} directly being the corresponding TD error.
It is therefore still optimizing for the invariant optimal policy.

(Part II)
{\it Eqn. \eqref{second_step} optimizes for the task-posterior $V_{M_j}^{*}(s)$.}

Let $\phi_j^\prime = \arg\min_{\phi_j}\|R_j^\prime (s, a, s^\prime) + \gamma \max_{a^\prime} A_{\phi_j}(s^\prime, a^\prime) - A_{\phi_j}(s, a)\|^2$,
i.e., assume Eqn. \eqref{first_step} optimizes to minimum the parameters that it has gradients on, and get $\phi_j^\prime$. Following the remark in Part I, we have 
\begin{equation}
\label{eqn:p2-asmpt1}
    A_{\phi_j^\prime}(s^\prime, a^\prime) = Q_{M_j^\prime}^*(s, a)
\end{equation}
from Q-learning on the shaped MDP $M_j^\prime$.

We also rearrange Eqn. \eqref{relation} with the adopted $\Phi(s) = V_{\phi_j}(s)$ to get:
\begin{equation}
\label{eqn:p2-from-invariance}
    Q_{M_j^\prime}^*(s, a) + V_{\phi_j}(s) = Q_{M_j}^*(s, a).
\end{equation}

Substituting Eqn. \eqref{eqn:p2-asmpt1} into \eqref{eqn:p2-from-invariance},
we get:
\begin{equation}
         Q_{M_j}^*(s, a) = Q_{M_j^\prime}^*(s, a) + V_{\phi_j}(s)
         = A_{\phi_j^\prime}(s, a) + V_{\phi_j}(s).\\
\end{equation}

Note that by definition,
\begin{equation}
    Q_{M_j}^*(s, a) = V_{M_j}^*(s) + A_{M_j}^{*}(s, a).
\end{equation}

So we have
\begin{equation}
\label{eqn:p2-penulti}
    V_{M_j}^*(s) + A_{M_j}^{*}(s, a) = A_{\phi_j^\prime}(s, a) + V_{\phi_j}(s).
\end{equation}

Taking $\max_a$ on both sides of Eqn. \eqref{eqn:p2-penulti}, we get
\begin{align}
    V_{M_j}^*(s) + \max_a A_{M_j}^{*}(s, a) = & \max_a A_{\phi_j^\prime}(s, a) + V_{\phi_j}(s) \notag\\
    V_{M_j}^*(s) = & \max_a A_{\phi_j^\prime}(s, a) + V_{\phi_j}(s),
    \label{eqn:p2-v}
\end{align}
where $\max_a A_{M_j}^{*}(s, a) = \max_a  Q_{M_j}^*(s, a) - V_{M_j}^*(s) = 0$ holds by definition. 

In this way, we transform the inaccessible $V_{M_j}^*(s)$ into the computable 
$\max_a A_{\phi_j^\prime}(s, a) + V_{\phi_j}(s)$, and
to adapt $V_{\phi_j}(s)$ to the task-posterior $V_{M_j}^*(s)$ one should minimize 
\begin{equation}
    \|V_{\phi_j}(s)) - \texttt{stop\_gradient} \big(\max_a A_{\phi_j^\prime}(s, a) + V_{\phi_j}(s) \big)\|^2,
\end{equation}
where we stop the gradients because the latter part should be treated as a scalar learning target.

Therefore, Eqn. \eqref{second_step} is indeed optimizing for the task-posterior $V_{M_j}^{*}(s)$.

\textbf{Remark:}
Here we assume $\phi_j$ is optimized to the final $\phi_j^\prime$. In practice this is not necessary nor desired, preventing fast adaptation.
Therefore, we take only one step of Eqn. \eqref{first_step}, and alternate between one step of Eqn. \eqref{first_step} and one step of Eqn. \eqref{second_step}, where one pair constitutes one update step in Fig. \ref{fig:find} and \ref{fig:maze} (left).

Also note that $A_{\phi_j}$ and $V_{\phi_j}$ may or may not share parameters, and $\phi_j^\prime$ only corresponds to the parameters that Eqn. \eqref{first_step} has gradients on, so we keep the separate notations of $\phi_j$ and $\phi_j^\prime$. From the above derivation, we can see that Eqn. \eqref{eqn:p2-v} holds for arbitrary $\phi_j$, so nothing is violated if some parameters of $\phi_j$ is updated by Eqn. \eqref{first_step}.
\end{proof}

\end{document}